\documentclass{article}
\usepackage{ijcai20}

\usepackage{times}
\usepackage{soul}
\usepackage{url}
\usepackage[hidelinks]{hyperref}
\usepackage[utf8]{inputenc}
\usepackage[small]{caption}
\usepackage{graphicx}
\usepackage{amsmath}
\usepackage{amssymb}
\usepackage{mathtools}
\usepackage{amsthm}
\usepackage{color}

\newtheorem{lemma}{Lemma}
\newtheorem*{lemma_new}{Lemma}
\newtheorem*{theorem_new}{Theorem}
\DeclareMathOperator*{\argmin}{argmin}

\usepackage{booktabs}
\usepackage{algorithm}
\usepackage{algorithmic}

\newtheorem{theorem}{Theorem}

\newcommand{\ea}[1]{\begin{equation}\begin{aligned}\centering #1 \end{aligned}\end{equation}}

\title{Off-policy Maximum Entropy Reinforcement Learning : Soft Actor-Critic with Advantage Weighted Mixture Policy(SAC-AWMP)}

\author{
Zhimin Hou$^{1 *}$\and
Kuangen Zhang$^2$\footnote{Equal contribution}\and
Yi Wan$^{3}$\and
Dongyu Li$^{1}$\and
Chenglong Fu$^2$\And
Haoyong Yu$^1$\\
\affiliations
$^1$ Department of Biomedical Engineering, National University of Singapore\\
$^2$ Department of Mechanical and Energy Engineering, Southern University of Science and Technology\\
$^3$ Reinforcement Learning and Artificial Intelligence Laboratory, University of Alberta\\
\emails
\{biehz, levyli, bieyhy\}@nus.edu.sg,
\{zhangkn, fucl\}@sustech.edu.cn,
\{wan6\}@ualberta.ca
}

\begin{document}

\maketitle

\begin{abstract}
The optimal policy of a reinforcement learning problem is often discontinuous and non-smooth. I.e., for two states with similar representations, their optimal policies can be significantly different. In this case, representing the entire policy with a function approximator (FA) with shared parameters for all states may be not desirable, as the generalization ability of parameters sharing makes representing discontinuous, non-smooth policies difficult. A common way to solve with this problem, known as Mixture-of-Experts, is to represent the policy as the weighted sum of multiple components, where different components perform well on different parts of the state space. Following this idea and inspired by a recent work called advantage-weighted information maximization, we propose to learn for each state weights of these components, so that they entail the information of the state itself and also the preferred action learned so far for the state. The action preference is characterized via the advantage function. In this case, the weight of each component would only be large for certain group of states whose representations are similar and preferred action representations are also similar. Therefore each component is easy to be represented. We call a policy parameterized in this way an \emph{Advantage Weighted Mixture Policy} (AWMP) and apply this idea to improve soft-actor-critic (SAC), one of the most competitive continuous control algorithm. Experimental results demonstrate that SAC with AWMP clearly outperforms SAC in four commonly used continuous control tasks and achieve stable performance across different random seeds. 
\end{abstract}

\section{Introduction}
Many interesting reinforcement learning(RL) problems has large state space such as go\cite{silver2017mastering}, atari games\cite{mnih2015human} and robotic manipulation control \cite{levine2016end,zhimin_2018_TII}. 
For these problems, it is impossible for the agent to visit all states and therefore tabular approaches in classic reinforcement learning would not be helpful \cite{sutton2018reinforcement,yi2019expectation_model}.  
It is therefore necessary for the agent to use function approximation (FA) to represent its value, policy or model and share parameters for different inputs. Sharing parameters gives the FA generalization ability, which loosely speaking, means that the FA's outputs are similar given similar inputs. The generalization ability makes it possible to generate reliable outputs for inputs that are not used to train the FA. Meanwhile, it also makes it hard to approximate functions that are discontinuous, or not smooth. 

A method called Mixture-of-Experts (MoE)\cite{xu1995alternative} was commonly used to deal with the above issue. 
This method approximates a function with a group of experts. 
Each expert only approximates the function locally. 
The rationale is that although the function to be approximated may be discontinuous and non-smooth for all possible inputs, it could still be smooth and continuous for a some inputs and therefore easy to be represented only for those inputs. As long as the ensemble of experts covers the entire input space, the approximated function can then be the weighted sum of these experts \cite{shazeer2017outrageously}. 
The only unsolved problem is the determination of weights of experts, for different inputs. A reasonable idea is each expert's weight should be high only for a part of the input space where the function is easy to be represented. 

In this paper, we use MoE to represent the agent's policy. 
For concreteness, let's call each expert a \emph{policy component} and the entire policy the \emph{mixture policy}. 
The weights of policy components are learned by a method called advantage-weighted information maximization, which cleverly assigns weights so that each policy component is simple to represent. This method was originally proposed in \cite{osa2019hierarchical} as a way to learn policy-over-options for temporal abstraction\cite{bacon2017option}. 
However, we note that this method by itself is a way to generate for each state a probability distribution and is not limited to the use of temporal abstraction. In this paper, we apply it to learn the weights of the mixture policy. We call a mixture policy whose weights are learned via advantage-weighted information maximization an \emph{advantage weighted mixture policy} (AWMP). 

In AWMP, given a state, the weights of policy components are generated by a neural network, called the \emph{prior network} that takes the given state as input. 
Parameters of the prior network are learned through maximizing the mutual information(MI)\cite{chen2016infogan,houthooft2016vime} between the state-action pair under the policy induced by the advantage function learned so far and the policy component sampled from the probability distribution induced by the weights of policy components. 
Due to the generalization ability of the prior network, for similar state-action pair inputs, the corresponding weights output would also likely to be similar. 
Meanwhile, in order to maximize the MI, state-action pairs whose representations are different by large degree are likely to produce different weights.  
In this way, each policy component would have high weight, only for a group of states whose representations are similar and representations of preferred actions for these states, learned so far, are also similar. 
For this reason, the policy component for these states is much simpler and requires less capacity to represent.
\begin{figure}[!t]
\centering
\includegraphics[width=3.5in]{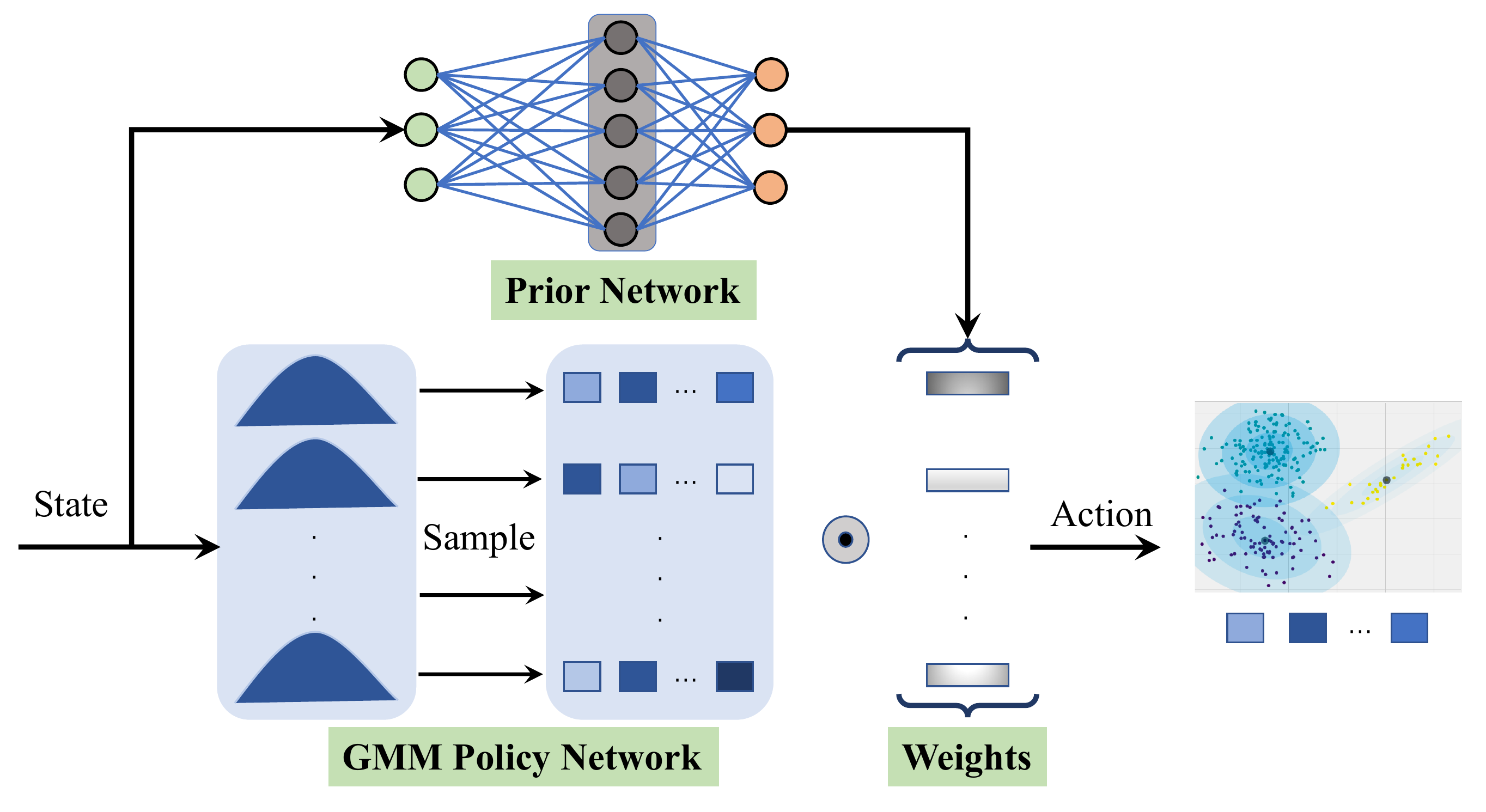} 
\caption{Systematic of AWMP}  
\label{skematic_gmm_policy}
\end{figure}

AWMP in general can be combined with any policy-based RL algorithm, including both state-of-the-art on-policy method PPO\cite{schulman2017proximal} and the most popular off-policy method TD3\cite{fujimoto2018addressing}. 
In this paper, we combine it with one of the state-of-the-art off-policy continuous control algorithm, soft-actor-critic(SAC)\cite{haarnoja2018soft}, which has achieved the start-of-the-art efficiency and stability performance on several continuous control Mujoco tasks and challenging real-world robotic tasks\cite{haarnoja2018soft_applications}. 
SAC aims to learn a stochastic Gaussian policy based on the maximum entropy objective through maximizing both the given reward and an augmented entropy term \cite{ziebart2008maximum}. 
Therefore, as in Figure~\ref{skematic_gmm_policy}, we propose a Gaussian mixture policy with several stochastic policy components as the mixture policy, each policy component estimates an independent Gaussian policy. 
The resulting algorithm is called soft-actor-critic with advantage-weighted mixture policy(SAC-AWMP). 
We show empirically the resulting algorithm clearly outperforms the standard SAC and TD3 in four commonly used continuous control domains, in terms of both the \emph{learning efficiency} and the \emph{stability}. 
The rest of the paper is arranged as follows: section.~\ref{sec:background} introduces the necessary background, including maximum entropy RL and mutual information, section.~\ref{sec:algorithm} first introduce the prior network and SAC-AWMP algorithm.   
In section.~\ref{sec:experiments} we empirically compare SAC-AWMP with the standard SAC, which shows the proposed SAC-AWMP could improve the performance of SAC.

\section{Background}\label{sec:background}
In this section, we define the notation and derive the soft policy iteration of maximum entropy RL. 

\subsection{Preliminaries}
The tasks with continuous state and action space addressed by RL generally is formulated as a MDP $\big(\mathcal{S}, \mathcal{A}, T, \mathcal{R}\big)$, consist of a state space $\mathcal{S}$, a action space $\mathcal{A}$, a state transition function $T: \mathcal{S} \times \mathcal{A} \rightarrow \mathcal{S}$ and a reward function $\mathcal{R}: \mathcal{S} \times \mathcal{A} \rightarrow \mathbb{R}$. 
At each environment step $t$, based on the state of environment $\boldsymbol{s}_t \in \mathcal{S}$, the agent select an action $\boldsymbol{a}_t \in \mathcal{A}$ generated by the policy $\pi(\boldsymbol{a}_t | \boldsymbol{s}_t): \mathcal{S} \rightarrow \mathcal{A}$, 
then the agent will receive a reward $R_{t+1}:  \mathbb{E}[R_{t+1}] = \mathcal{R}(\boldsymbol{s}_t, \boldsymbol{a}_t)
$ and the environment transit to next state $\boldsymbol{s}_{t+1} \in \mathcal{S}$.  
A trajectory denotes as $\tau = (\boldsymbol{s}_0, \boldsymbol{a}_0, \cdots, \boldsymbol{s}_T)$ given an initial state distribution $d_0$, which starts from an initial state $\boldsymbol{s}_0 \thicksim d_0(\boldsymbol{s}_0)$ and follows the action under the policy $\pi(\cdot | \cdot)$.  
The standard RL objective to learn the policy is maximizing the received expected return $\mathcal{J} = \mathbb{E}_{\tau}[G_t]$, return $G_t = \sum_{i=t}^T R_t$ denotes the cumulative reward of one episode from $t$ to the terminal time step $T$.  

\subsection{Maximum entropy RL}
Compared to the standard RL objective, the maximum entropy RL objective augmented with the the entropy of the stochastic policy is formulated as:
\begin{equation}
    \mathcal{J} (\pi) = \int \int \mathcal{P}_{\pi}(\boldsymbol{a}, \boldsymbol{s}) \left( Q^{\pi} (\boldsymbol{s}, \boldsymbol{a}) + \alpha \mathcal{H}(\pi(\cdot | \boldsymbol{s}))\right) d\boldsymbol{a}d\boldsymbol{s} 
\end{equation}
where $Q^{\pi} (\boldsymbol{s}, \boldsymbol{a}) : \mathcal{S} \times \mathcal{A} \rightarrow \mathbb{R}$ denotes the action value function to approximate the expected return. $\mathcal{P}_{\pi}(\boldsymbol{a}, \boldsymbol{s}) = d_0^{\pi}(\boldsymbol{s}) \pi(\boldsymbol{a}|\boldsymbol{s})$ denotes the probability of state-action pair tracking the trajectory induced by policy $\pi$. 
$d_0^{\pi}(\boldsymbol{s})$ denotes the visitation frequency of $\boldsymbol{s}$. 
$\alpha$ denotes the temperature to balance the importance of the stochasticity of the optimal policy against the cumulative reward.   

\subsection{Soft policy iteration}
To learn the optimal maximum entropy policy with the convergence guarantee, 
soft policy iteration is derived similar as in \cite{haarnoja2018soft}, which repeats soft policy evaluation and soft policy improvement alternately. 
In soft policy evaluation iteration, given a fixed policy $\pi$, the soft action value $Q^{\pi}$ is calculated iteratively via a designed soft Bellman backup operator 
$\mathcal{T}^{\pi}$ as: 
\begin{equation}\label{equ:policy_evaluation_update}
\begin{aligned}
    \mathcal{T}^{\pi} Q^{\pi}(\boldsymbol{s}_t, \boldsymbol{a}_t) &\triangleq \mathcal{R}^{\pi}(\boldsymbol{s}_t, \boldsymbol{a}_t)\\
    &+ \gamma \mathbb{E}_{\boldsymbol{s}_{t +1}, \boldsymbol{a}_{t +1} \thicksim \mathcal{P}_{\pi}}
\left[
Q^{\pi}(\boldsymbol{s}_{t +1}, \boldsymbol{a}_{t+1})
\right]
\end{aligned} 
\end{equation}
where $ \mathcal{R}^{\pi}(\boldsymbol{s}_t, \boldsymbol{a}_t) \triangleq \mathcal{R} (\boldsymbol{s}_t, \boldsymbol{a}_t) + \mathbb{E}_{\boldsymbol{s}_{t+1} \thicksim \mathcal{P}_{\pi}}[\mathcal{H}(\pi(\cdot | \boldsymbol{s}_{t+1}))]$ denotes the entropy reward; $\gamma$ denotes a practical discount factor. 

\begin{lemma}
\label{lemma:soft_policy_evaluation}
(Soft Policy Evaluation). 
Consider the soft Bellman backup operator $\mathcal{T}^{\pi}$ and the initial soft action value $Q^{\pi_0}$: $\mathcal{S} \times \mathcal{A} \rightarrow \mathbb{R}$  with $|\mathcal{A}| < \infty$. 
Define $Q^{k+1} = \mathcal{T}^{\pi} Q^{k}$, as $k \rightarrow \infty$, $Q^k$ will converge to the soft action value $Q^{\pi}$ of $\pi$. 
\end{lemma}
\begin{proof}
With the assumption $|\mathcal{A}| < \infty$ to guarantee the entropy augmented reward $\mathcal{R}^{\pi}(\boldsymbol{s}_t, \boldsymbol{a}_t)$ is bounded, then the convergence of soft policy evaluation updated as in Equation.(\ref{equ:policy_evaluation_update}) can be proofed as standard policy evaluation\cite{sutton2018reinforcement}.  
\end{proof}

In soft policy improvement iteration, the policy $\pi_{k+1}$ is updated to minimize the Kullback-Leibler(KL)
divergence between the next policy and the target distribution induced by the soft action value function $Q^{\pi_k}$, as:
\begin{equation}\label{equ:soft_policy_improvement}
    \pi_{k+1} = \argmin_{\pi^{\prime} \in \Pi} \rm{D_{KL}} \left( \pi^{\prime}(\cdot|\boldsymbol{s}_t) \Big| \Big| \frac{\exp(Q^{\pi_k}(\boldsymbol{s}_t, \cdot)) }{\Phi^{\pi_k} (\boldsymbol{s}_t)} \right) 
\end{equation}
\begin{lemma}
\label{lemma:soft_policy_improvement}
(Soft Policy Improvement). 
Consider $\pi_k \in \Pi$ and let $\pi_{k+1}$ be the optimizer of the minimization objective in Equation.~(\ref{equ:soft_policy_improvement}).  
Then $Q^{\pi_{k+1}}(\boldsymbol{s}_t, \boldsymbol{a}_t) \ge Q^{\pi_k}(\boldsymbol{s}_t, \boldsymbol{a}_t)$ for all the $(\boldsymbol{s}_t, \boldsymbol{a}_t) \in \mathcal{S} \times \mathcal{A}$ with assumption $|\mathcal{A}| < \infty$. 
\end{lemma}
\begin{proof}
See Supplementary Material.~\ref{proof:soft_policy_improvement}. 
\end{proof}

The full soft policy iteration(Theorem.~1) is derived similar as SAC \cite{haarnoja2018soft} to replace the soft policy evaluation and soft policy improvement. 
To perform the continuous control tasks, the soft action value and policy of SAC need to be estimated by the function approximation. 
\begin{theorem}
\label{theorem:soft_policy_iteration}
(Soft Policy Iteration).  
Repeat soft policy evaluation and soft policy improvement alternately, 
start from any initial policy $\pi \in \Pi$ will converges to a optimal policy $\pi^{*}$ with $Q^{\pi^{*}}(\boldsymbol{s}_t, \boldsymbol{a}_t) \ge Q^{\pi}(\boldsymbol{s}_t, \boldsymbol{a}_t)$ for all $\pi \in \Pi$ and all $(\boldsymbol{s}_t, \boldsymbol{a}_t) \in \mathcal{S} \times \mathcal{A}$ with assumption $|\mathcal{A}| < \infty$. 
\end{theorem}
\begin{proof}
See Supplementary Material.~\ref{proof:soft_policy_iteration}. 
\end{proof}

\subsection{Mutual information}
MI maximization has been researched to learn the interpretable representation\cite{chen2016infogan} and achieve effective exploration for continuous control RL\cite{houthooft2016vime}. 
MI $\mathcal{I}(X, Y)$ denotes the amount of information between random variable $X$ and $Y$, calculated as:
\begin{equation}
    \mathcal{I}(X, Y) = \mathcal{H}(Y) - \mathcal{H}(Y|X)
\end{equation}
where $\mathcal{H}(Y)$ and $\mathcal{H}(Y|X)$ represents the entropy and conditional entropy, respectively. 

\section{SAC-AWMP}\label{sec:algorithm}
Some previous work aims to learn the hierarchical policy with the latent variable through dividing the state space. 
Instead, in this paper, the state-action space is divided corresponding to the mode of advantage function. 
Firstly, the prior network is implemented to learn the weights of AWMP based on a advantage-weighted information maximization objective.  
Secondly, we introduce how to learn the AWMP on \emph{off-policy} data with the designed maximum entropy objective. 

\subsection{Prior network}  
The prior network $\mathcal{P}(\boldsymbol{h}|\boldsymbol{s}, \boldsymbol{a}; \boldsymbol{\eta})$ with a \emph{softmax} output layer is parameterized by $\boldsymbol{\eta}$ to obtain the weights $\boldsymbol{h} \in \mathbb{R}^{\mathcal{O}}$, $\mathcal{O}$ is the number of policy components in AWMP. 
To maximize the MI of state-action pairs and weights, the parameters $\boldsymbol{\eta}$ are updated through minimizing a regularized objective\cite{krause2010discriminative}, as: 
\begin{equation}\label{equ-mi-objective}
    \mathcal{J}_{\boldsymbol{\eta}} (\boldsymbol{\eta}) = \ell(\boldsymbol{\eta}) - \zeta \mathcal{I} (\boldsymbol{h}, (\boldsymbol{s}, \boldsymbol{a});\boldsymbol{\eta}) 
\end{equation}
where the regularization term $\ell(\boldsymbol{\eta})$ generally is calculated via $\rm{D_{KL}} \big(\mathcal{P}(\boldsymbol{h}|\Tilde{\boldsymbol{s}}, \Tilde{\boldsymbol{a}}; \boldsymbol{\eta}) || \mathcal{P}(\boldsymbol{h}|\boldsymbol{s}, \boldsymbol{a}; \boldsymbol{\eta}) \big)$ to penalize the instability against the perturbation. 
$\zeta$ is a coefficient to balance the performance. 
The improvement of regularization term trick has been verified in many learning representation tasks\cite{osa2019hierarchical}. 

The MI in Equation.~(\ref{equ-mi-objective}) denotes as: 
\begin{equation}
    \mathcal{I} (\boldsymbol{h}, (\boldsymbol{s}, \boldsymbol{a}); \boldsymbol{\eta}) = \mathcal{H}(\boldsymbol{h}; \boldsymbol{\eta}) - \mathcal{H}(\boldsymbol{h} |\boldsymbol{s}, \boldsymbol{a}; \boldsymbol{\eta})
\end{equation}
where the entropy $\mathcal{H}(\boldsymbol{h}; \boldsymbol{\eta})$ is estimated by: 
\begin{equation}
    \mathcal{H}(\boldsymbol{h};\boldsymbol{\eta}) = - \int \mathcal{P}(\boldsymbol{h}; \boldsymbol{\eta}) \log \mathcal{P}(\boldsymbol{h}; \boldsymbol{\eta}) d \boldsymbol{h}
\end{equation}
where $\mathcal{P}(\boldsymbol{h}; \boldsymbol{\eta})$ denotes the probability of weights derived from the probability density $\mathcal{P}_{A} (\boldsymbol{s}, \boldsymbol{a})$, as:
\begin{equation}\label{equ:weight_probability}
    \mathcal{P}(\boldsymbol{h}; \boldsymbol{\eta}) = \int \mathcal{P}_{A} (\boldsymbol{s}, \boldsymbol{a}) \mathcal{P}(\boldsymbol{h} | \boldsymbol{s}, \boldsymbol{a}; \boldsymbol{\eta}) d\boldsymbol{a} d\boldsymbol{s}
\end{equation}
where $\mathcal{P}_{A} (\boldsymbol{s}, \boldsymbol{a})$ denotes the probability density of state-action pair $(\boldsymbol{s}, \boldsymbol{a})$ induced by a policy $\pi^A$ based on the mode of advantage function as $\digamma (A^{\pi} (\boldsymbol{s}, \boldsymbol{a}))$. 
We consider a formulation as $\digamma (A^{\pi} (\boldsymbol{s}, \boldsymbol{a})) = \exp{(A^{\pi} (\boldsymbol{s}, \boldsymbol{a}))} / \Phi(\boldsymbol{s}, \boldsymbol{a})$, which can meet the requirement that the state-action pair with larger advantage value given a higher probability.
$\Phi(\boldsymbol{s}, \boldsymbol{a})$ denotes the partition function. 

Likewise, the conditional entropy $\mathcal{H}(\boldsymbol{h} |\boldsymbol{s}, \boldsymbol{a}; \boldsymbol{\eta})$ is estimated by:
\begin{equation}
\resizebox{.99\linewidth}{!}{$
    \displaystyle
    \mathcal{H}(\boldsymbol{h} |\boldsymbol{s}, \boldsymbol{a}; \boldsymbol{\eta}) = \int \mathcal{P}_{A} (\boldsymbol{s}, \boldsymbol{a}) \mathcal{P}(\boldsymbol{h} | (\boldsymbol{s}, \boldsymbol{a}); \boldsymbol{\eta}) \log \mathcal{P}(\boldsymbol{h} | (\boldsymbol{s}, \boldsymbol{a}); \boldsymbol{\eta}) d\boldsymbol{a} d\boldsymbol{s}
$}
\end{equation}

In practice, it is not available to estimate the probability density $\mathcal{P}_{A} (\boldsymbol{s}, \boldsymbol{a})$ from past experience, to solve this issue, the advantage-weighted importance sampling approach is introduced. 
$\mu (\boldsymbol{a} | \boldsymbol{s})$ denotes the behavior policy to generate the experience. 
We assume that the state distribution only changes sufficiently small resulting in $d^{\pi^A}(\boldsymbol{s}) \approx d^{\mu}(\boldsymbol{s})$, 
The advantage-weighted importance sampling weights are calculated as: 
\begin{equation}
\varpi (\boldsymbol{s}, \boldsymbol{a}) = \frac{\mathcal{P}_{A} (\boldsymbol{s}, \boldsymbol{a})}{\mathcal{P}_{\mu} (\boldsymbol{s}, \boldsymbol{a})} = \frac{d^{\pi^A}(\boldsymbol{s}) \pi^A(\boldsymbol{s}, \boldsymbol{a}))}{d^{\mu}(\boldsymbol{s})\mu (\boldsymbol{a} | \boldsymbol{s})} \approx
\frac{\digamma (A^{\pi} (\boldsymbol{s}, \boldsymbol{a}))}{\mu (\boldsymbol{a} | \boldsymbol{s})}
\end{equation}

To improve the training stability, the advantage-weighted importance sampling weights are normalized as:
\begin{equation}\label{density-function}
    \hat{\varpi} (\boldsymbol{s}, \boldsymbol{a}) = \frac{\frac{\digamma (A^{\pi} (\boldsymbol{s}, \boldsymbol{a}))}{\mu (\boldsymbol{a} | \boldsymbol{s})}}{\sum_i^N \frac{\digamma (A^{\pi} (\boldsymbol{s}_i, \boldsymbol{a}_i))}{\mu (\boldsymbol{a}_i | \boldsymbol{s}_i)}} 
\end{equation}
where $N$ is size of samples in replay buffer $\mathcal{D}$.     
Based on advantaged-weights importance sampling in Equation.~(\ref{density-function}), then probability of weights in Equation.~(\ref{equ:weight_probability}) can be estimated by:
\begin{equation}
    \hat{\mathcal{P}}(\boldsymbol{h}; \boldsymbol{\eta}) = \mathbb{E}_{(\boldsymbol{s}, \boldsymbol{a}) \thicksim \mathcal{D}}[\hat{\varpi} (\boldsymbol{s}, \boldsymbol{a}) \mathcal{P}(\boldsymbol{h} | (\boldsymbol{s}, \boldsymbol{a}); \boldsymbol{\eta})]
\end{equation}
Therefore, the parameterized entropy is estimated by:
\begin{equation}
    \hat{\mathcal{H}}(\boldsymbol{h};\boldsymbol{\eta}) = - \int \hat{\mathcal{P}}(\boldsymbol{h}; \boldsymbol{\eta}) \log \hat{\mathcal{P}}(\boldsymbol{h}; \boldsymbol{\eta}) d \boldsymbol{h}
\end{equation}
Likewise, the parameterized conditional entropy $\hat{\mathcal{H}}(\boldsymbol{h} | (\boldsymbol{s}, \boldsymbol{a}); \boldsymbol{\eta})$ is estimated by:
\begin{equation}\label{entropy-objective}
\begin{aligned}
    \mathbb{E}_{(\boldsymbol{s}, \boldsymbol{a}) \thicksim \mathcal{D}}[ \varpi_A (\boldsymbol{s}, \boldsymbol{a}) \mathcal{P}(\boldsymbol{h} | (\boldsymbol{s}, \boldsymbol{a}); \boldsymbol{\eta}) \log \mathcal{P}(\boldsymbol{h} | (\boldsymbol{s}, \boldsymbol{a}); \boldsymbol{\eta})]
\end{aligned}
\end{equation}

\subsection{AWMP}
The AWMP network $\pi(\boldsymbol{a}|\boldsymbol{s}; \boldsymbol{\theta})$ is comprised of $\mathcal{G}$ policy components parameterized by $\boldsymbol{\theta}$. 
The maximum entropy objective to learn the AWMP denotes as: 
\begin{equation}\label{entropy-objective}
\resizebox{.99\linewidth}{!}{$
    \displaystyle
    \mathcal{J}(\pi) = \int \int d^{\pi}(\boldsymbol{s}) \pi(\boldsymbol{a}|\boldsymbol{s}, \boldsymbol{\theta}) \left(Q^{\pi} (\boldsymbol{s}, \boldsymbol{a}) + \alpha_{\pi} \mathcal{H}(\pi(\cdot | \boldsymbol{s}, \boldsymbol{\theta}))\right) d\boldsymbol{a}d\boldsymbol{s}
$}
\end{equation}
where $Q^{\pi} (\boldsymbol{s}, \boldsymbol{a})$ denotes the soft action value function induced by AWMP $\pi$;  
$\pi(\boldsymbol{a}|\boldsymbol{s}; \boldsymbol{\theta}) = \sum_{g \in \mathcal{G}} \rho(g |\boldsymbol{s}) \pi_g(\boldsymbol{a} | \boldsymbol{s}, g; \boldsymbol{\theta})$, here $\mathcal{G}$ is set possible value of $g$;  
$\rho(g |\boldsymbol{s}) $ denotes a \emph{gating} policy. 
$\pi_g(\boldsymbol{a}|\boldsymbol{s}, g; \boldsymbol{\theta}) \in \mathbb{R}^{|\mathcal{A}|}$ denotes a single policy component given $g$, a stochastic Gaussian policy; 
$\alpha_{\pi}$ denotes the entropy temperature to control the stochasticity of the AWMP. 

At each environment step, the action is sampled from the AWMP and executed to interact with environment. 
Given the state $\boldsymbol{s}$, 
the softmax \emph{gating} policy $\rho(g |\boldsymbol{s})$ is calculated by: 
\begin{equation}\label{gating-policy}
    \rho(g | \boldsymbol{s}) = \frac{\exp({Q_{\mathcal{G}}^{\pi} (\boldsymbol{s}, g)})} {\sum_{g \in \mathcal{G}}\exp \big( {Q_{\mathcal{G}}^{\pi} (\boldsymbol{s}, g)}\big)} 
\end{equation}
where $Q_{\mathcal{G}}^{\pi} (\boldsymbol{s}, g)$ named as soft option value \cite{bacon2017option} represents the conditional expectation of return following a given policy component $\pi_g (\boldsymbol{a}|\boldsymbol{s}, g; \boldsymbol{\theta})$, as: 
\begin{equation}\label{option-value-equ}
\begin{aligned}
    &Q_{\mathcal{G}}^{\pi} (\boldsymbol{s}, g) = \mathbb{E}_{\boldsymbol{a} \thicksim \pi_g(\cdot | \boldsymbol{s}, g)}[R_t| \boldsymbol{s}_t = \boldsymbol{s}, \boldsymbol{g}_t = g] \\
    &= \int \pi_g(\boldsymbol{a}|\boldsymbol{s}, g; \boldsymbol{\theta}) \big( Q^{\pi}(\boldsymbol{s}, \boldsymbol{a}) - \alpha_{\pi} \log \pi_g(\boldsymbol{a}|\boldsymbol{s}, g; \boldsymbol{\theta}) \big) d\boldsymbol{a}
\end{aligned}
\end{equation}

The soft state value induced by AWMP $\pi$ is derived as:  
\begin{equation}
\begin{aligned}
    V^{\pi} (\boldsymbol{s}) =& \mathbb{E}_{g \thicksim \rho( \cdot | \boldsymbol{s})} \big[ Q^{\pi}_{\mathcal{G}}(\boldsymbol{s}, g) - \alpha_g \log \rho(g |\boldsymbol{s}) \big]  \\
    = &\int \rho(g |\boldsymbol{s}) \big( Q^{\pi}_{\mathcal{G}}(\boldsymbol{s}, g) - \alpha_g \log \rho(g |\boldsymbol{s}) \big) d g\\
\end{aligned}
\end{equation}
where $\alpha_g$ denotes the entropy temperature of the \emph{gating} policy.  
The advantage value is calculated as:
\ea{
    A^{\pi}(\boldsymbol{s}, \boldsymbol{a}) = Q^{\pi}(\boldsymbol{s}, \boldsymbol{a}) - V^{\pi}(\boldsymbol{s})
}

The function approximators are applied to estimate both the soft state-value $V^{\pi}(\boldsymbol{s})$ and soft action value $Q^{\pi}(\boldsymbol{s}, \boldsymbol{a})$. 
Based on the soft policy iteration theorem, 
the soft state value network $V(\boldsymbol{s}; \boldsymbol{\psi})$ with parameters $\boldsymbol{\psi}$ and the soft action value network $Q(\boldsymbol{s}, \boldsymbol{a}; \boldsymbol{w})$ with parameters $\boldsymbol{w}$ can be learned alternatively through stochastic gradient descent(SGD). 

The soft state value network $V(\boldsymbol{s}; \boldsymbol{\psi})$ is trained through minimizing the following error:  
\begin{equation}
\begin{aligned}
\mathcal{J}_V(\psi) = \mathbb{E}_{\boldsymbol{s}_t \thicksim \mathcal{D}}\left [\frac{1}{2}\Big(V(\boldsymbol{s}_t; \boldsymbol{\psi}) - \hat{V}(\boldsymbol{s}_t)
 \Big) ^2\right]
\end{aligned}
\end{equation}
where state $\boldsymbol{s}_t$ is sampled from the replay buffer $\mathcal{D}$ and the target value $\hat{V}(\boldsymbol{s}_t)$ is calculated by:
\begin{equation}
\resizebox{.99\linewidth}{!}{$
    \displaystyle
    \mathbb{E}_{h \thicksim \rho} \left[ \mathbb{E}_{\boldsymbol{a} \thicksim \pi_h}
    \left [ 
    Q(\boldsymbol{s}_t, \boldsymbol{a}; \boldsymbol{w}) - \alpha_{\pi} \log \pi_h(\boldsymbol{a}|\boldsymbol{s}_t, h; \boldsymbol{\theta}) \right] - \alpha_h \log \rho(h |\boldsymbol{s})  \right]
$}
\end{equation}
where the corresponding action $\boldsymbol{a}$ is sampled from the current AWMP. 
In particular, we implement two independent soft action value network $\left\{ Q (\boldsymbol{s}, \boldsymbol{a}; \boldsymbol{w}_1), Q (\boldsymbol{s}, \boldsymbol{a}; \boldsymbol{w}_2) \right\}$ and select the minimal one to calculate the target value $\hat{V}(\boldsymbol{s}_t)$, which has been demonstrated to reduce the effect of positive bias and improve the sample-efficiency in previous value-based work \cite{fujimoto2018addressing,zhang2019teach}.  
The parameter $\boldsymbol{w} \in \{ \boldsymbol{w}_1, \boldsymbol{w}_2 \}$ soft action value function is updated through minimizing the soft Bellman residual error: 
\begin{equation}
    \mathcal{J}_Q(\boldsymbol{w}) = \mathbb{E}_{(\boldsymbol{s}_t, \boldsymbol{a}_t) \thicksim \mathcal{D}}\left[\frac{1}{2}\Big(Q(\boldsymbol{s}_t, \boldsymbol{a}_t; \boldsymbol{w}) - \hat{Q}(\boldsymbol{s}_t, \boldsymbol{a}_t)\Big)^2 \right]
\end{equation}
where $(\boldsymbol{s}_t, \boldsymbol{a}_t)$ is sampled from replay buffer $\mathcal{D}$; the target value $\hat{Q}(\boldsymbol{s}_t, \boldsymbol{a}_t)$ is calculated by: 
\begin{equation}
    \hat{Q}(\boldsymbol{s}_t, \boldsymbol{a}_t) = \mathcal{R} (\boldsymbol{s}_t, \boldsymbol{a}_t) + \gamma \mathbb{E}_{\boldsymbol{s}_{t+1} \thicksim \mathcal{P}} \big[V(\boldsymbol{s}_{t+1}; \Bar{\psi} )\big]
\end{equation}
where $V(\boldsymbol{s}_{t+1}; \Bar{\psi} )$ is calculated from the target soft state value network parameterized by $\Bar{\psi}$. 

Instead of minimizing the objective in Equation.~(\ref{entropy-objective}) through the gradient backpropagating, we apply the likelihood ratio gradient estimator to learn the AWMP based on the off-policy data in replay buffer\cite{williams1992simple,haarnoja2018soft}.  
The objective is rewritten as: 
\begin{equation}\label{objective-function}
    \mathcal{J}_{\pi}(\boldsymbol{\theta}, \boldsymbol{w}) = \mathbb{E}_{\boldsymbol{s}_t \thicksim \mathcal{D}} \left[ \rm {D_{KL}} \left ( \pi(\cdot|\boldsymbol{s}_t; \boldsymbol{\theta}) || \frac{\exp(Q(\boldsymbol{s}_t, \cdot; \boldsymbol{w}, \boldsymbol{\theta})) }{\Phi(\boldsymbol{s}_t;\boldsymbol{w})} \right ) \right]
\end{equation}
which is utilized to minimize the expected KL-divergence with the target density induced by $Q(\boldsymbol{s}_t, \boldsymbol{a})$; 
$\boldsymbol{s}_t$ is sampled from the replay buffer $\mathcal{D}$ and $\boldsymbol{a}$ is derived from the current AWMP $\pi(\cdot|\boldsymbol{s}_t; \boldsymbol{\theta})$; 
the weights $h$ for AWMP is derived from the prior network $\mathcal{P}(\cdot|\boldsymbol{s}_t, \boldsymbol{a}; \boldsymbol{\eta})$; 
$\Phi(\boldsymbol{s}_t;\boldsymbol{w}) = \int \pi(\boldsymbol{a}|\boldsymbol{s}_t; \boldsymbol{\theta}) \exp(Q(\boldsymbol{s}_t, \boldsymbol{a}; \boldsymbol{w})) \boldsymbol{da}$ denotes the partition function. 
The soft action value network can be differentiated similar as the \emph{deterministic policy gradient}(DPG) theorem \cite{silver2014deterministic}.  
Therefore, a transformation trick is implemented to represent the AWMP with the weight $h$ from the prior network, as:
\ea{
    \boldsymbol{a}_t = \mathcal{F}(\epsilon_t; \boldsymbol{s}_t, \boldsymbol{\theta}) = \sum_i^{\mathcal{O}} h_i f_i (\epsilon_t; \boldsymbol{s}_t, \boldsymbol{\theta})
}
where the action $\boldsymbol{a}_t^i$ is derived from $i$th policy component and $h_i$ is the $i$th element of the weight. 
$\epsilon_t$ is the input noise sampled from a fixed Gaussian distribution $\mathcal{N}(\mathbf{0}, \mathbf{I}) \in \mathbb{R}^{|\mathcal{A}| \times \mathcal{O}}$. 
The objective function in Equation.~(\ref{objective-function}) is rewritten as:
\begin{equation}
\begin{aligned}
\mathcal{J}_{\pi}(\boldsymbol{\theta}, \boldsymbol{w}) = \mathbb{E}_{\boldsymbol{s}_t \thicksim \mathcal{D}, \epsilon_t \thicksim \mathcal{N}}
\big[ 
\log \pi(\mathcal{F}(\epsilon_t; \boldsymbol{s}_t, \boldsymbol{\theta}) |\boldsymbol{s}_t)
\\
- Q(\boldsymbol{s}_t, \mathcal{F}(\epsilon_t; \boldsymbol{s}_t, \boldsymbol{\theta}); \boldsymbol{w}) + \log \Phi(\boldsymbol{s}_t;\boldsymbol{w})
\big]
\end{aligned}
\end{equation}

The gradient of partition function $\Phi(\cdot)$ can be omitted due to be independent of parameters $\boldsymbol{\theta}$, the gradient with respect to the parameters $\boldsymbol{\theta}$ is calculated as: 
\ea{
&\hat{\nabla}_{\theta} \mathcal{J}_{\pi}(\boldsymbol{\theta}, \boldsymbol{w}) = \nabla_{\theta} \log \pi(\boldsymbol{a}_t|\boldsymbol{s}_t; \boldsymbol{\theta})
\\
&+ (\nabla_{\boldsymbol{a}_t}\log \pi(\boldsymbol{a}_t|\boldsymbol{s}_t; \boldsymbol{\theta}) - \nabla_{\boldsymbol{a}_t} Q(\boldsymbol{s}_t, \boldsymbol{a}_t; \boldsymbol{w}) \nabla_{\theta} \mathcal{F}(\epsilon_t; \boldsymbol{s}_t, \boldsymbol{\theta})
}

\subsection{Practice training}
Each policy component of the AWMP network will output the action from unbounded Gaussian distribution $\pi_g(\cdot)$, 
in practice, an invertible squashing function ($\tanh$) is applied to bound the Gaussian samples $\boldsymbol{u} \in \mathbb{R}^{|\mathcal{A}|}$ elementwise. 
Therefore, the action $\boldsymbol{a} = \tanh{(\boldsymbol{u})}$ will be restricted in $[-1, 1]^{|\mathcal{A}|}$. 
The density of action $\boldsymbol{a}$ induced by the Jacobian of the transformation, denotes as:
\begin{equation}
\begin{aligned}
    \pi_g(\boldsymbol{a}|\boldsymbol{s}, g) &= \mu(\boldsymbol{u}|\boldsymbol{s}, g) \Big|\det \big (\frac{d \boldsymbol{a}}{d \boldsymbol{u}} \big) \Big|^{-1} 
\end{aligned}
\end{equation}
Therefore, the action sampled from the AWMP weighted by the weights from prior network is still restricted in $[-1, 1]^{|\mathcal{A}|}$. 

Our proposed algorithm is summarised in Algorithm.~\ref{Soft Option-Critic}. 
At each gradient step, the soft action-value network, the soft state-value network and the AWMP network are trained on the mini-batch \emph{off-policy} samples $\mathcal{B}_{off}$ from replay buffer $\mathcal{D}$. 
To improve the training stability,  
the prior network is trained on \emph{semi off-policy} data, therefore we sample the mini-batch $\mathcal{B}_{on}$ from the most recent $M$ samples generated by the most recent behavior policies. 
Additionally, to sufficiently update the AWMP, a target soft state-action value network $Q(\boldsymbol{s}, \boldsymbol{a}; \Bar{\boldsymbol{w}})$ is implemented to obtain a less frequent \emph{gating} policy. 
The exponentially moving average with a smoothing constant $\tau_V$ and $\tau_Q$ is applied to update the target networks, respectively. 
\begin{algorithm}[!t]
  \begin{algorithmic}[1]
    \STATE \textbf{Input:} Number of policy components $\mathcal{O}$, size of replay buffer $\mathcal{D}$, size of mini-batch $\mathcal{B}_{on}$ and $\mathcal{B}_{off}$
    \STATE \textbf{Initialize:} $\mathcal{P}(\cdot|\cdot~, \cdot~; \boldsymbol{\eta})$, $\pi(\cdot~|\cdot~; \boldsymbol{\theta})$;  $Q(\cdot~, \cdot~ ;\boldsymbol{\omega})$, $Q(\cdot~, \cdot~ ;{\Bar{\boldsymbol{\omega}}})$;
    $V_{\psi}(\cdot~; \boldsymbol{\psi})$, $V_{\Bar{\psi}}(\cdot~; \Bar{\boldsymbol{\psi}})$
    \FOR {each iteration}
    \FOR {each environment step}
    \STATE $\boldsymbol{a}_t = \pi (\boldsymbol{a}_t|\boldsymbol{s}_t; \boldsymbol{\theta}) = \sum_{g \in \mathcal{G}} \rho( g | \boldsymbol{s}_t) \pi_g (\boldsymbol{a}_t | \boldsymbol{s}_t, g; \boldsymbol{\theta})$
    \STATE $R_{t+1}$, $\boldsymbol{s}_{t+1} \thicksim T (\boldsymbol{s}_{t+1} | \boldsymbol{s}_t, \boldsymbol{a}_t)$
    \STATE $\mathcal{D} \leftarrow \mathcal{D} \bigcup (\boldsymbol{s}_t, \boldsymbol{a}_t, R_{t+1}, \boldsymbol{s}_{t+1})$
    \ENDFOR
    \FOR {each prior network update step}
    \STATE Semi off-policy sample $\mathcal{B}_{on}$ from $ \mathcal{D}$
    \STATE $\boldsymbol{\eta} \leftarrow \boldsymbol{\eta} - \alpha_{\eta} \hat{\nabla}_{\boldsymbol{\eta}} \mathcal{J}_{\eta} (\eta)$
    \ENDFOR
    \FOR {each gradient step}
    \STATE Off-policy samples $\boldsymbol{s}_i \in \mathcal{B}_{off}$ from $ \mathcal{D}$
    \STATE $\boldsymbol{a}_i \thicksim \pi(\cdot | \boldsymbol{s}_i; \boldsymbol{\theta})$
    \STATE $h_{(i)} = \mathcal{P}(\cdot|\boldsymbol{s}_i, \boldsymbol{a}_i; \boldsymbol{\eta})$
    \STATE $\boldsymbol{\psi} \leftarrow \boldsymbol{\psi} - \alpha_{\psi}\hat{\nabla}_{\boldsymbol{\psi}}J_V(\boldsymbol{\psi})$
    \STATE $\boldsymbol{\omega} \leftarrow \boldsymbol{\omega} - \alpha_{\boldsymbol{\omega}}\hat{\nabla}_{\boldsymbol{\omega}}J_Q(\boldsymbol{\omega})$
    \STATE {$\boldsymbol{\theta} \leftarrow \boldsymbol{\theta} - \alpha_{\theta}\hat{\nabla}_{\boldsymbol{\theta}}J_{\pi}(\boldsymbol{\theta})$}
    \STATE Target network update
    \STATE {$\Bar{\boldsymbol{\psi}} \leftarrow \tau_V \boldsymbol{\psi} + (1 - \tau_V)\Bar{\boldsymbol{\psi}}$}
    \STATE {$\Bar{\boldsymbol{\omega}} \leftarrow \tau_Q\boldsymbol{\omega} + (1 - \tau_Q)\Bar{\boldsymbol{\omega}}$}
    \ENDFOR
    \ENDFOR
  \end{algorithmic}
  \caption{SAC-AWMP}
  \label{Soft Option-Critic} 
\end{algorithm}

\section{Experiments}\label{sec:experiments}
Our proposed SAC-AWMP is evaluated to understand sample complexity and stability compared to the previous state-of-the-art RL algorithms on four commonly used continuous control tasks of OpenAI Gym (see Fig.~\ref{fig:environments})\cite{todorov2012mujoco,brockman2016openai}. 
\begin{figure}[!t]
    \centering
    \includegraphics[width=\linewidth]{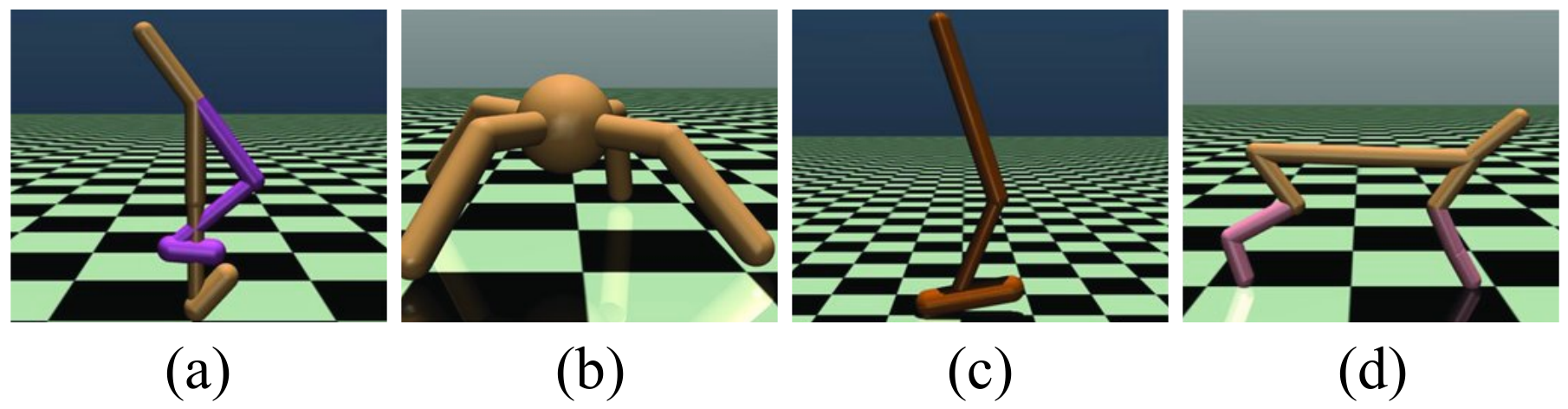}
    \caption{MuJoCo tasks. (a) Walker2d-v2, (b) Ant-v2, (c) Hopper-v2, (d) HalfCheetah-v2.}
    \label{fig:environments}
\end{figure}

\subsection{Settings}
The sample complexity of off-policy RL such as TD3 and SAC compared to the state-of-the-art on-policy algorithm PPO \cite{schulman2017proximal} has been done in TD3 and SAC\cite{fujimoto2018addressing,haarnoja2018soft}. 
Therefore, in this paper, our proposed SAC-AWMP is only implemented compared to TD3 and SAC, which are implemented with the provided code of authors. 
Most specifically, each policy component of SAC-AWMP has the same network architecture with SAC. 
Each algorithm for all the tasks is trained for five trials with different seed(0, 1, 2, 3, 4), each trial with 1 million steps, and the expected return is estimated via ten evaluation episodes every 1000 experiment steps. 
All hyperparameters used in our experiments refers to the original papers of SAC and TD3\cite{fujimoto2018addressing,haarnoja2018soft}, which are listed in Supplementary Material~B
and the source code is available on github\footnote{Code: \url{https://github.com/hzm2016/SAC-AWMP.git}}.  

\subsection{Results and Comparisons}
As shown in Figure~\ref{results_gmm_policy}, 
the solid curves represents the mean of the average evaluation and the shaded region corresponds to half a standard deviation of the evaluation over five seed. 
SAC-AWMP with four separate policy components can outperform than SAC in term of learning efficiency and stability(on Ant-v2, Walker2d-v2 and Hopper-v2), Halfcheetah-2 can be solved easily by all RL methods.
However, on the harder task Ant-v2, SAC-AWMP will outperform than SAC and TD3 largely. 
Obviously, SAC-AWMP and SAC can achieve better stability than TD3 on all the four tasks without any hyperparameters tuning. 
\begin{figure}[!t]
\centering
\includegraphics[width=3.6in]{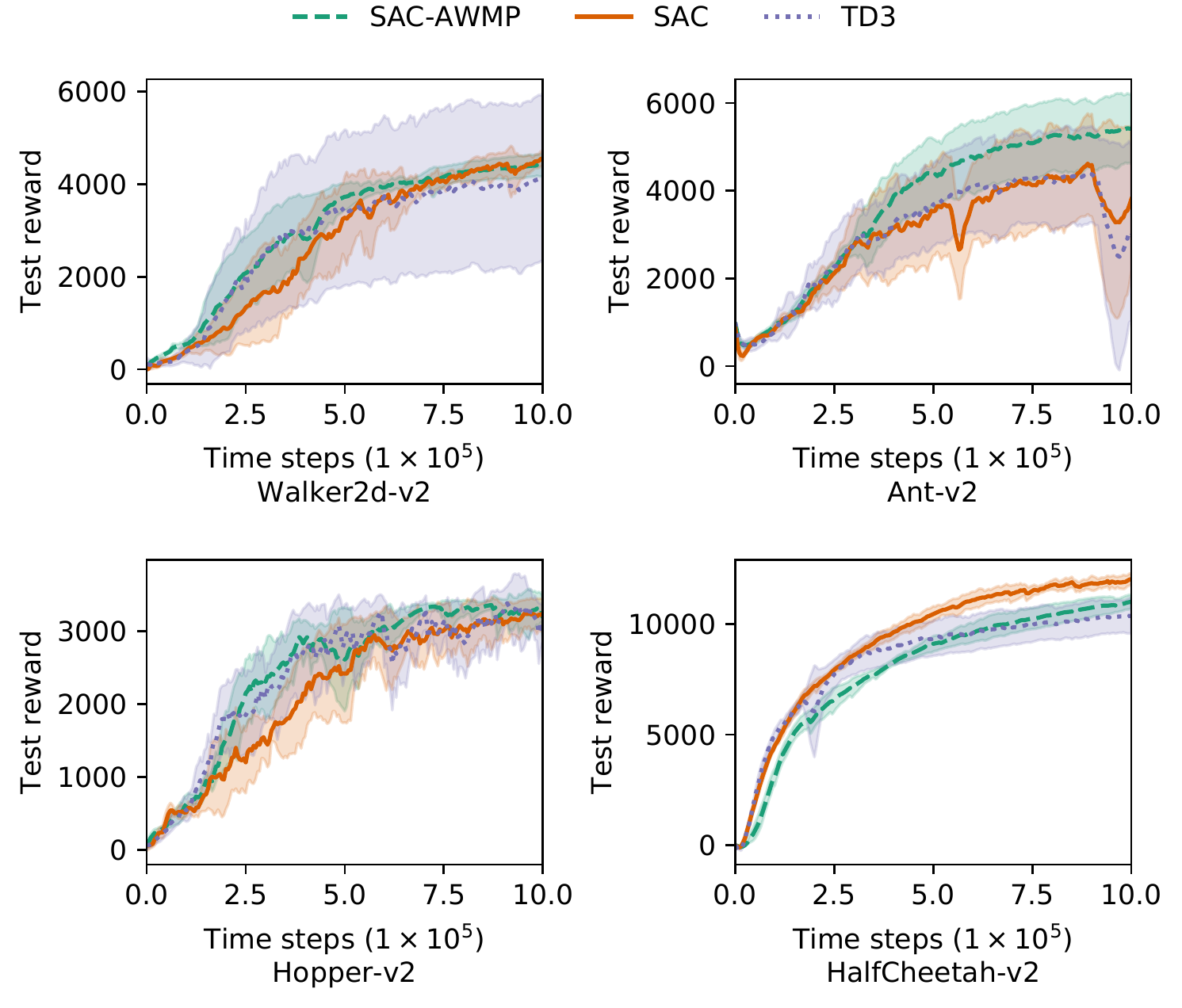}
\caption{Learning curves for the Mujoco continuous control tasks. 
Entropy temperature term $\alpha_{\pi} = 0.2$ and $\alpha_g = 0.001$. 
Number of policy components $\mathcal{O} = 4$.} 
\label{results_gmm_policy}
\end{figure}

\begin{figure}[!t]
\centering
\includegraphics[width=3.6in]{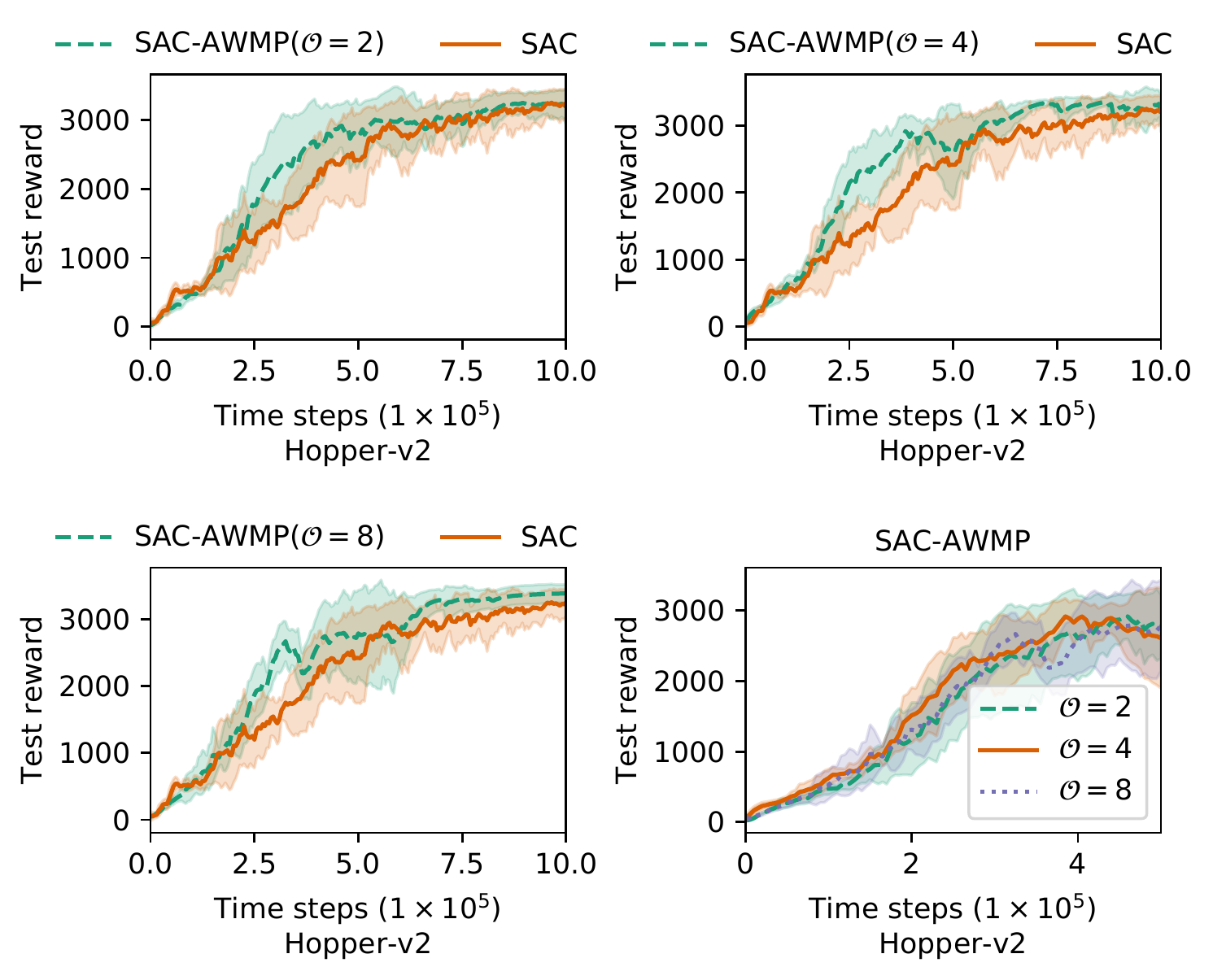} 
\caption{Learning curves for SAC-AWMP with different number of policy components. 
Entropy temperature term $\alpha_{\pi} = 0.2$ and $\alpha_g = 0.001$. } 
\label{aw-gmm-different-number}
\end{figure}
The number of policy components in SAC-AWMP need to be given, which is similar as the number of options and latent variables in hierarchical RL(HRL) \cite{bacon2017option,osa2019hierarchical,zhang2019dac}. 
How to discover the meaningful policy components and option policies corresponding to each latent variable is a long standing open question.  
The number of policy components should be tasks-dependent, which has not been investigated clearly. 
In previous HRL work, for all the continuous control tasks, two or four option policies were tested in \cite{osa2019hierarchical,zhang2019dac}. 
In this paper, our proposed SAC-AWMP is implemented with four different number of policy components $\mathcal{O} = 1, 2, 4, 8 \cdots$. 
When only one single policy component is applied, the proposed SAC-AWMP will degenerate as SAC \cite{haarnoja2018soft}. 
As shown in Figure~\ref{aw-gmm-different-number}, 
SAC-AWMP can outperform the SAC with three different number of policy components, and SAC-AWMP with 8 policy components could result in small variance during training.  
\begin{figure}[!t]
\centering
\includegraphics[width=3.6in]{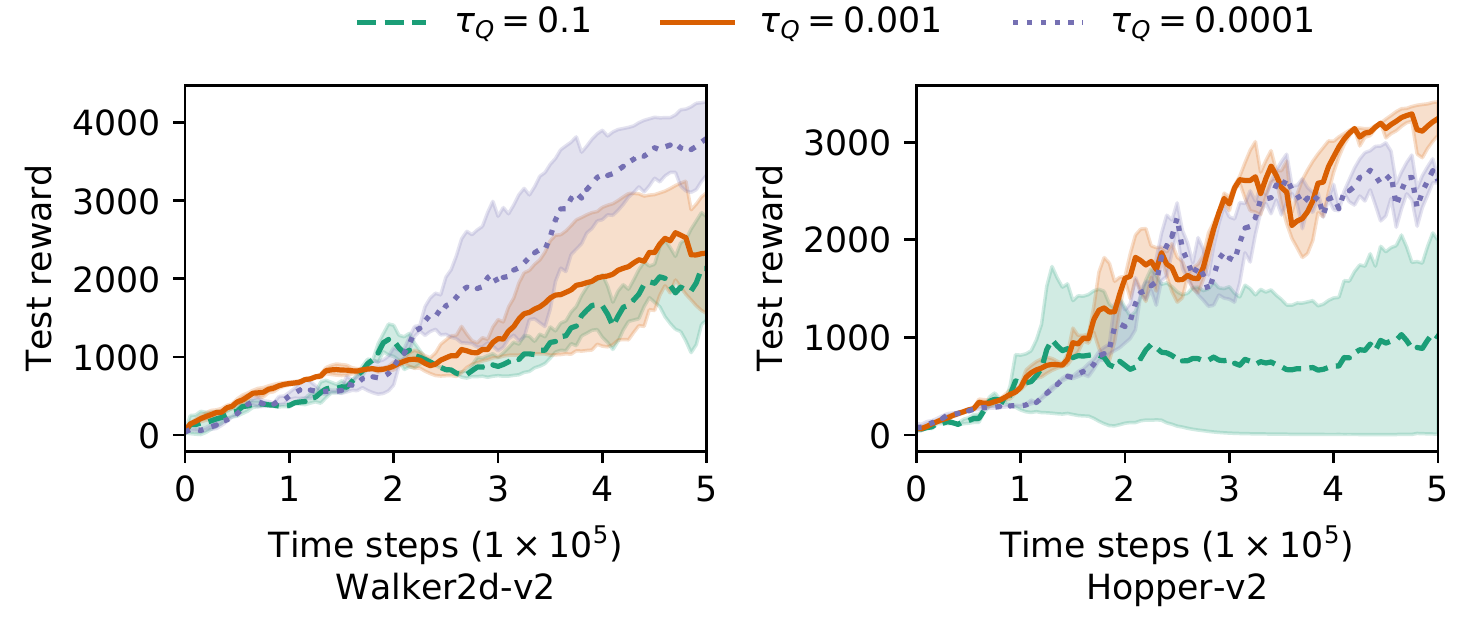}
\caption{Learning curves for SAC-AWMP with different smoothing coefficient.} 
\label{aw-gmm-target-smoothing-coefficient}
\end{figure}

The performance of maximum entropy RL largely depends on the entropy temperature, which is replaced by the reward scale in \cite{haarnoja2018soft}. 
The automating entropy adjustment varying across the different tasks and different learning stages proposed in \cite{haarnoja2018soft_applications}, however, it has not achieved too much improvement compared to fixed entropy given in \cite{haarnoja2018soft}. 
In this paper, the entropy temperature for each task is fixed same as in \cite{haarnoja2018soft}($\alpha_{\pi} = 0.2$ see Figure \ref{results_gmm_policy}).  
Target network is a commonly used trick to slowly track the changing value updated via a smoothing coefficient\cite{lillicrap2015continuous},   
which has largely improved learning stability of RL algorithms. 
As depicted in Algorithm 1, the AWMP could be learned on \emph{off-policy} data, however the prior network need to be learned on \emph{semi off-policy} data(generated by most recent policies). 
In addition to the target network for soft state value function, 
we implement a independent target network for soft action value to derive the \emph{gating} policy. 
The smoothing coefficient $\tau_Q = 0.001$ is applied for the above experiments, additionally, as Figure~\ref{aw-gmm-target-smoothing-coefficient}, we test other two different smoothing coefficients$\tau_Q = {0.1, 0.0001}$. 
Large smoothing coefficient $\tau_Q$ may result in instability and divergence, but small value will lead to slower learning. 

\section{Conclusion and discussion}
In this article, we proposed a soft actor-critic with advantage weighted mixture policy(SAC-AWMP), an off-policy maximum entropy RL algorithm. 
Without any specific hyperparameters tuning, we empirically demonstrate that the
proposed SAC-AWMP with wights learned via advantage-weighted information maximization can achieve more smooth policy approximation and stable learning than TD3,  
and improve the sample-efficiency performance of the typical SAC on three Mujoco tasks.  

Compared to the typical SAC with single stochastic Gaussian policy,  
the AWMP hold the promise to solve the complex tasks with high dimensional continuous state and action space or the real-world tasks with hierarchical structures. 
Actually our proposed AWMP can combine with any policy-gradient methods, such as PPO and TD3.  
Additionally, in this paper, the prior network could only be learned via 'semi off-policy' data. 
For better sample-efficiency and applicability, further investigation could be done in these directions. 

\section*{Acknowledgments}
This work was supported by Agency for Science, Technology and Research, Singapore, under the National Robotics Program, with A*star SERC Grant No.: 192 25 00054.

\bibliographystyle{named}
\bibliography{ijcai20,my_paper}

\appendix

\section{Proofs}

\subsection{Lemma 2}\label{proof:soft_policy_improvement}
\begin{lemma_new}
(Soft Policy Improvement). 
Consider $\pi_k \in \Pi$ and let $\pi_{k+1}$ be the optimizer of the minimization problem in Equation(\ref{equ:soft_policy_improvement}). 
Then $Q^{\pi_{k+1}}(\boldsymbol{s}_t, \boldsymbol{a}_t) \ge Q^{\pi_k}(\boldsymbol{s}_t, \boldsymbol{a}_t)$ for all the $(\boldsymbol{s}_t, \boldsymbol{a}_t) \in \mathcal{S} \times \mathcal{A}$ with $|\mathcal{A}| < \infty$. 
\end{lemma_new}

\begin{proof}
Let $\pi_k$, $Q^{\pi_k}$ and $V^{\pi_k}$ denote the old policy, the soft state-action value and soft state value, and then $\pi_{k+1}$ as new policy is defined as Equation.~(\ref{equ:soft_policy_improvement}), rewritten as: 
\begin{equation}
    \pi_{k+1}(\cdot | \boldsymbol{s}_t) = \argmin_{\pi^{'} \in \Pi} \mathcal{J}_{\pi_k}(\pi^{'}(\cdot | \boldsymbol{s}_t)) 
\end{equation}
With $\pi_{k+1} \in \Pi$, it must be satisfied that $\mathcal{J}_{\pi_k}(\pi_{k+1} (\cdot | \boldsymbol{s}_t)) \leq \mathcal{J}_{\pi_k}(\pi_k (\cdot | \boldsymbol{s}_t))$. 
Hence
\begin{equation}
\begin{aligned}
    &\mathbb{E}_{\boldsymbol{a}_t \thicksim \pi_{k+1}}
    \left [\log \pi_{k+1}(\boldsymbol{a}_t|\boldsymbol{s}_t) - Q^{\pi_k}(\boldsymbol{s}_t, \boldsymbol{a}_t) + \log {\Phi^{\pi_k} (\boldsymbol{s}_t)} \right]\\
    &\leq 
    \mathbb{E}_{\boldsymbol{a}_t \thicksim \pi_{k}}
    \left [\log \pi_k(\boldsymbol{a}_t|\boldsymbol{s}_t) - Q^{\pi_k}(\boldsymbol{s}_t, \boldsymbol{a}_t) + \log {\Phi^{\pi_k} (\boldsymbol{s}_t)} \right]
\end{aligned}
\end{equation}
Since partition function $\Phi^{\pi_k} (\boldsymbol{s}_t)$ only depends on the state, the inequality reduces to:
\begin{equation}\label{equ:q-value-bound}
\begin{aligned}
    &\mathbb{E}_{\boldsymbol{a}_t \thicksim \pi_{k+1}}
    \left [Q^{\pi_k}(\boldsymbol{s}_t, \boldsymbol{a}_t) - \log \pi_{k+1}(\boldsymbol{a}_t|\boldsymbol{s}_t) \right] \ge V^{\pi_k}(\boldsymbol{s}_t) 
\end{aligned}
\end{equation}
Then, consider the soft Bellman equation: 
\begin{equation}
\begin{aligned}
    Q^{\pi_k}(\boldsymbol{s}_t, \boldsymbol{a}_t) &= \mathcal{R} (\boldsymbol{s}_t, \boldsymbol{a}_t) + \gamma \mathbb{E}_{\boldsymbol{s}_{t+1} \thicksim \mathcal{P}}[V^{\pi_k}(\boldsymbol{s}_{t+1})]
    \\
    &\leq \mathcal{R} (\boldsymbol{s}_t, \boldsymbol{a}_t) + \gamma \mathbb{E}_{\boldsymbol{s}_{t+1} \thicksim \mathcal{P}}\left[\mathbb{E}_{\boldsymbol{a}_t \thicksim \pi_{k+1}}\left [Q^{\pi_k}(\boldsymbol{s}_t, \boldsymbol{a}_t)\right.\right.
    \\
    & \left.\left. - \log \pi_{k+1}(\boldsymbol{a}_t|\boldsymbol{s}_t) \right ] \right]
    \\
    \vdots
    \\
    &\leq Q^{\pi_{k+1}}(\boldsymbol{s}_t, \boldsymbol{a}_t)
\end{aligned}
\end{equation}
where we expand the $Q^{\pi_k}(\boldsymbol{s}_t, \boldsymbol{a}_t)$ repeatedly by applying the soft Bellman equation and the bound in Equation.~(\ref{equ:q-value-bound}). Finally convergence to $Q^{\pi_{k+1}}(\boldsymbol{s}_t, \boldsymbol{a}_t)$ follows Lemma 1. 
\end{proof}

\subsection{Theorem 1}\label{proof:soft_policy_iteration}
\begin{theorem_new}
(Soft Policy Iteration).  
Repeat soft policy evaluation and soft improvement policy alternately, 
start from any initial policy $\pi \in \Pi$ will converges to optimal policy $\pi^{*}$ with $Q^{\pi^{*}}(\boldsymbol{s}_t, \boldsymbol{a}_t) \ge Q^{\pi}(\boldsymbol{s}_t, \boldsymbol{a}_t)$ for all $\pi \in \Pi$ and all the $(\boldsymbol{s}_t, \boldsymbol{a}_t) \in \mathcal{S} \times \mathcal{A}$ with $|\mathcal{A}| < \infty$. 
\end{theorem_new}

\begin{proof}
Let $\pi_k$ denote the policy at iteration $k$. Based on Lemma 2, the sequence $Q^{\pi_k}$ is monotonically increasing. 
The sequence $\pi_k$ will converge to some $\pi^{*}$ due to $Q^{\pi}$ is bounded above for $\pi \in \Pi$. 
At convergence, it must be case that $\mathcal{J}_{\pi^{*}}(\pi^{*} (\cdot | \boldsymbol{s}_t)) \leq \mathcal{J}_{\pi^{*}}(\pi (\cdot | \boldsymbol{s}_t))$ for all $\pi \in \Pi$. 
Based on the proof of Lemma 2, $Q^{\pi^{*}}(\boldsymbol{s}_t, \boldsymbol{a}_t) \ge Q^{\pi}(\boldsymbol{s}_t, \boldsymbol{a}_t)$ for all $(\boldsymbol{s}_t, \boldsymbol{a}_t) \in \mathcal{S} \times \mathcal{A}$. 
Hence, it must be case that $\pi^{*}$ is optimal in $\Pi$.  
\end{proof}

\section{Details of Experiments}\label{stylefiles}
\begin{table}[htbp]
\caption{Mujoco Environments Settings}\label{sample-table}
\begin{center}
\begin{tabular}{lcc}
\hline
\multicolumn{1}{l}{Description}  & \multicolumn{1}{l}{Action Dimensions} & \multicolumn{1}{l}{Entropy $\alpha_{\pi}$}
\\
\hline
Ant-v2      &  8 & 0.2 \\
HalfCheetah-v2  & 6  & 0.2\\
Walker2d-v2         & 6  & 0.2\\
Hopper-v2     & 3  & 0.2\\
\hline
\end{tabular}
\end{center}
\end{table}

\begin{table}[htbp]
\caption{Hyper-parameters of SAC}
\label{sample-table}
\begin{center}
\begin{tabular}{lcc}
\hline
\multicolumn{1}{l}{Description}  & \multicolumn{1}{l}{Symbol} & \multicolumn{1}{l}{ Value}
\\ 
\hline
Batch size for critic         &   & 100\\
Number of hidden layers       &   & (400, 400)\\
Activation function                     &  & Relu, Relu, tanh\\
Target smoothing coefficient       & $\tau_V$ & 0.005\\
Learning rate                 &   & 3e-4\\
Gradient Steps               &   & 1   \\
Replay buffer size            &   & 1e6 \\
Entropy term               &$\alpha_{\pi}$&  0.2\\
Optimizer                  &  & Adam \\
Discount factor            & $\gamma$ & 0.99 \\
\hline 
\end{tabular}
\end{center}
\end{table}

\begin{table}[htbp]
\caption{Additional hyper-parameters of SAC-AWMP}\label{sample-table}
\begin{center}
\begin{tabular}{lcc}
\hline
\multicolumn{1}{l}{Description}  & \multicolumn{1}{l}{Symbol} & \multicolumn{1}{l}{ Value}
\\
\hline 
Batch size for critic         &   & 100\\
Batch size for policy           &   & 200($\mathcal{O}=2$)\\
                                  &   & 400($\mathcal{O}=4$)\\
Batch size for prior network      &   & 50\\
Target smoothing coefficient        & $\tau_Q$ & 0.001\\
Prior network update steps                & $M$ & 5000\\
Learning rate    & & 3e-4\\
Noise for MI regularization        & & 0.04\\
Coefficient for MI                 & & 0.1\\
Entropy term               &$\alpha_g$&  0.001\\
\hline
\end{tabular}
\end{center}
\end{table}
\end{document}